\pdfoutput=1

\documentclass[11pt]{article}

\usepackage[]{EMNLP2023}

\usepackage{times}
\usepackage{latexsym}
\usepackage{algorithmic}  
\usepackage{amsmath}  
\usepackage{algorithm} 
\usepackage[algo2e]{algorithm2e} 
\usepackage[utf8]{inputenc} 
\usepackage[T1]{fontenc}    
\usepackage{hyperref}       
\usepackage{url}            
\usepackage{booktabs}       
\usepackage{amsfonts}       
\usepackage{nicefrac}       
\usepackage{microtype}      
\usepackage{xcolor}         
\usepackage{makecell}
\usepackage{xcolor,colortbl}
\usepackage{pifont,dsfont}
\usepackage{wrapfig,lipsum,booktabs}
\usepackage[export]{adjustbox}
\usepackage{caption}

\usepackage[most]{tcolorbox}
\definecolor{block-gray}{gray}{0.85}
\newtcolorbox{myquote}{colframe=black,boxrule=1pt,
colback=white,grow to right by=-1mm,grow to left by=-1mm,
boxsep=0pt,breakable}

\newtcolorbox{inside_myquote}{boxrule=0pt,
colback=block-gray,grow to right by=3mm,grow to left by=3mm,
top=0pt,bottom=0pt}

\usepackage{times}
\usepackage{soul}
\usepackage{url}
\usepackage{amsmath}
\usepackage{amsthm}
\usepackage{booktabs}
\usepackage{algorithm}
\usepackage{algorithmic}
\usepackage{caption}
\usepackage{makecell} 
\usepackage{graphicx}
\usepackage{subcaption}

\usepackage{amssymb,stmaryrd}

\usepackage{xcolor,colortbl}

\newcommand{\hide}[1]{}

\newtheorem{theorem}{Theorem}

\usepackage[T1]{fontenc}

\usepackage[utf8]{inputenc}

\usepackage{microtype}

\usepackage{inconsolata}

\title{
Robust Conformal Consensus: Multi-Agent LLM-as-a-Judge Interval Evaluation with Conformal Prediction
}

\author{
  Lihui Liu \\
  Wayne State University\\
  Detroit, Michigan, USA\\
  \texttt{hw6926@wayne.edu} \\
}

\begin{document}

\maketitle

\begin{abstract}

LLM-as-a-Judge has emerged as a promising paradigm for evaluating natural language generation. However, the uncertainty associated with such evaluations remains largely unexplored, which limits their reliability in real-world applications. Although conformal prediction offers a principled framework for uncertainty quantification, existing approaches typically apply it to a single LLM judge, overlooking the variability introduced by using different LLM evaluators. In this work, we propose a robust uncertainty estimation framework for multi-agent LLM-as-a-Judge evaluation. Our approach constructs conformal prediction intervals for LLM-based scores from multiple LLMs. By considering intervals from different LLM judges, we obtain more stable and reliable uncertainty estimates. Extensive experiments demonstrate that our method produces valid prediction intervals with coverage guarantees, and that interval-based aggregation across multiple judges leads to more stable evaluation outcomes.

\vspace{-0.7\baselineskip}
\end{abstract}


\section{Introduction}

Large language models (LLMs) are increasingly being adopted as automated evaluators for natural language generation (NLG) tasks, a setting commonly known as LLM-as-a-Judge. Prior studies show that LLM-based evaluators correlate well with human assessments and achieve competitive performance on established evaluation metrics such as ROUGE~\cite{lin-2004-rouge}, BLEU~\cite{blue}, and BERTScore~\cite{BERTScore}. In addition to metric-level agreement, LLM judges provide substantial advantages in terms of flexibility and scalability: they can be readily adapted to diverse evaluation objectives and offer a cost-effective alternative to human annotation. As a result, LLM-as-a-Judge has been successfully applied across a broad range of domains, including healthcare~\cite{healthcare}, question answering~\cite{qa}, cybersecurity~\cite{safety}, and trafficking detection~\cite{traffic}.

Despite their advantages, evaluations based on a single LLM judge can be unreliable. LLM outputs are inherently random, so a single score may be biased or unstable~\cite{unstable}. This unreliability is especially concerning in high-stakes areas such as healthcare and finance ~\cite{healthcare}. Although techniques like prompt engineering or fine-tuning can reduce uncertainty, LLM judges may still produce inconsistent or overly confident evaluations.
Conformal prediction ~\cite{conformal_prediction} offers a clear way to address this issue. It produces prediction intervals that indicate a range where the true score is likely to fall, with statistical guarantees. Because it works without assumptions about the model or data and requires only a small calibration set, conformal prediction is well suited for measuring uncertainty in LLM-based evaluations.

However, existing conformal prediction methods for LLM-based evaluation are predominantly designed for a single LLM judge and are applied to the scores produced by one model in isolation. In practice, different LLMs often assign noticeably different scores to the same input due to variations in training data, architectures, alignment strategies, and inherent inductive biases. As a consequence, uncertainty estimates derived from a single LLM can be highly model-dependent and may fail to reflect the true variability of the evaluation process. This limitation can lead to misleadingly narrow prediction intervals, overconfident assessments, and poor robustness when the choice of LLM judge changes. To obtain more reliable and stable uncertainty estimates, it is therefore necessary to explicitly account for inter-LLM variability by incorporating scores from multiple LLM judges. Leveraging multiple LLMs as independent evaluators enables the construction of uncertainty estimates that better capture disagreement across judges, resulting in more robust prediction intervals and more trustworthy LLM-as-a-Judge evaluations.

In this paper, we introduce a robust interval adaptation framework that can be integrated with existing conformal prediction methods. Instead of relying on a single LLM judge, our approach explicitly incorporates multiple LLMs and recalibrates prediction intervals to account for inter-LLM variability. We apply the proposed framework to nine conformal prediction techniques to quantify uncertainty in rating-based LLM-as-a-Judge evaluations, where each method produces a prediction interval over LLM-generated scores. For each approach, we assess both {efficiency}, measured by the average interval width, and {coverage}, defined as the probability that the true rating lies within the predicted interval.
Our analysis shows that the quality of prediction intervals depends on the LLM used and the size of the calibration set. Using multiple LLMs provides more stable and reliable intervals than relying on a single model.
These findings emphasize the value of multi-LLM, uncertainty-aware evaluation for more robust assessments.

In summary, our contributions are:
\begin{itemize}
    \item We propose a {multi-LLM, robust framework} that quantifies uncertainty in LLM-as-a-judge evaluations using conformal prediction for rating-based tasks.
    \item We run experiments with multiple LLMs, showing that our method yields more stable uncertainty estimates than using a single LLM.
\end{itemize}

\section{Preliminaries}\label{preliminary}

\paragraph{LLM-as-a-Judge.}  
Large language models (LLMs) are increasingly used as automatic evaluators for natural language generation, a paradigm commonly called {LLM-as-a-judge}. In rating-based evaluations, an LLM assigns a numerical score $y_0$ to a candidate text $x$ according to a predefined scale, such as a Likert scale ~\cite{elangovan2024beyond}. Formally, given a prompt $p$, an LLM judge $M$ produces
\begin{equation}
M(p, x) = (z, y_0),
\end{equation}
where $z$ represents the token-level logits and $y_0$ is the final score. To obtain $y_0$, we focus on the logits corresponding to the tokens that encode the discrete rating options (e.g., 1–5). In addition to scalar scores, LLMs can handle more complex evaluation setups, such as pairwise comparisons or ranking tasks, by scoring multiple candidate outputs and deriving relative judgments based on these scores. This flexibility makes LLMs versatile evaluators across diverse NLG scenarios.

\paragraph{Conformal Prediction.}  
Conformal prediction is a flexible, model-agnostic technique for measuring uncertainty in predictions ~\cite{conformal_prediction}. Instead of giving a single point estimate, it produces a {prediction interval} that is statistically guaranteed to contain the true value with a pre-specified probability. Two key features make it appealing for LLM evaluation: it works {post-hoc} without modifying the model, and it does not rely on assumptions about the data distribution.  

In this paper, we use {split conformal prediction}, which estimates uncertainty using a separate calibration set. For each calibration example, we compute a {non-conformity score} that quantifies how far the model’s prediction $\hat{y}$ is from the true value $y$:
\begin{equation}
s(z, y) = |\hat{y} - y|,
\end{equation}
where $\hat{y} = f(z)$ is the predicted score derived from model logits $z$.  

Given a desired error rate $\alpha$, we take the $\lceil (n+1)(1-\alpha)\rceil$-th largest non-conformity score $\hat{q}$ from the calibration set to define the prediction interval for a new test input $z_\text{test}$:
\begin{equation}
C(z_\text{test}, \hat{y}_\text{test}) = [\hat{y}_\text{test} - \hat{q}, \hat{y}_\text{test} + \hat{q}].
\end{equation}
This interval satisfies a formal coverage guarantee, meaning that the true score $y_\text{test}$ will lie inside it at least $1-\alpha$ of the time, assuming calibration and test points are exchangeable ~\cite{conformal_prediction}:
\begin{equation}
1-\alpha \le \mathbb{P}\big(y_\text{test} \in C(z_\text{test}, \hat{y}_\text{test})\big) \le 1-\alpha + \frac{1}{n+1}.
\end{equation}
By constructing such intervals, conformal prediction allows us to quantify the uncertainty of LLM-generated ratings in a principled and interpretable way.


\section{Methodology}\label{single}

\begin{figure*}[h!]
    \centering
    \includegraphics[width=0.9\textwidth]{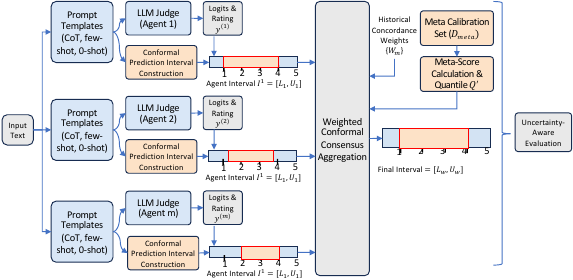} 
    \caption{Overview of uncertainty quantification in rating-based multi-agent LLM-as-a-Judge evaluation. We apply conformal prediction to multiple LLM judges to construct prediction intervals for their ratings. These intervals are then aggregated, and the width of the resulting interval is used as a measure of evaluation uncertainty.}
    \label{example} 
\end{figure*}

Existing conformal prediction approaches for LLM-based evaluation typically apply uncertainty quantification to a single judging agent, treating the LLM’s output as fixed given one prompt and one response. This setting overlooks the well-documented sensitivity of LLM judgments to prompt wording, instruction style, and sampling randomness, effectively collapsing diverse judging behaviors into a single agent view. For example, when the same generated text is evaluated by two judging agents instantiated through semantically equivalent Likert-scale prompts, such as “Rate the overall quality on a scale from 1 to 5” or “How good is this response? (1 = very poor, 5 = excellent)”, different LLMs may assign different ratings (e.g., 3 vs. 4), despite evaluating identical content. Applying conformal prediction to only one such agent produces a prediction interval that reflects agent-specific variability rather than intrinsic evaluation uncertainty, leading to unstable and potentially misleading uncertainty estimates.

To solve this problem, in this paper, we study uncertainty quantification for multi-agent LLM-as-a-judge in discrete rating-based evaluations (e.g., Likert scales) using conformal prediction. Instead of relying on a single judging agent, we instantiate multiple agents through semantically equivalent prompts. Each agent independently produces a judgment, which is then transformed into a conformal prediction interval. These agent-level intervals are subsequently aggregated to yield a final, uncertainty-aware evaluation that captures model uncertainty. An overview of the workflow is shown in Figure ~\ref{example}.

\subsection{LLM Rating Extraction and Conformal Prediction}

We focus on evaluation tasks where ratings are discrete, such as Likert scales (1--5). For each input text, we collect responses from $m$ different LLM judges, where each judge independently provides a single-turn rating. This setup captures variability across models, which is important since different LLMs may assign different scores to the same text.  

From each LLM response, we extract the token logits corresponding to candidate rating labels. In practice, outputs may not strictly follow a fixed format, so we adopt a simple rule-based strategy: identify the rating token at its most frequent position across responses from the same LLM. Once the position is located, we extract the log probabilities for all possible labels and merge tokens with equivalent meanings (e.g., ``two'' and ``2'') to maintain semantic consistency.  

This process produces, for each LLM $j$, a $K$-dimensional feature vector:
\[
\mathbf{z}^{(j)} \in \mathbb{R}^{K},
\]
representing logits for each rating label. Combined with the observed rating $y^{(j)}$, each pair $(\mathbf{z}^{(j)}, y^{(j)})$ is used as input to conformal prediction. We assume that pairs from different LLMs are exchangeable, allowing uncertainty to be estimated across the ensemble of judges.  

\paragraph{Constructing Prediction Intervals.}  
For each LLM, we generate a prediction interval for its rating. Intervals are preferred over sets because they naturally respect the ordinal structure of Likert ratings. To accommodate different modeling strategies, we apply a variety of conformal methods.  

For regression-based approaches, where logits often exhibit heteroscedasticity, we use Conformalized Quantile Regression (CQR) ~\cite{CQR} and its variants, including asymmetric CQR, CHR ~\cite{CHR}, LVD ~\cite{LVD}, Boosted CQR~\cite{boost}, Boosted LCP~\cite{boost}, and R2CCP~\cite{r2cpp}. For ordinal classification, ratings are treated as ordered categories, and we apply ordinal conformal methods such as Ordinal APS~\cite{aps} and Ordinal Risk Control~\cite{ordinal_risk_control}, using softmax probabilities as class likelihoods. Raw logits are retained for regression-based methods. Each LLM $j$ thus produces a conformal prediction interval $\mathcal{I}^{(j)}_{\alpha}$ with coverage $1-\alpha$.  

\paragraph{Boundary Adjustment.}  
In rating-based evaluations, the labels are discrete numbers (for example, 1 to 5), but the prediction intervals produced by conformal prediction are often continuous values. Continuous boundaries can be confusing or hard to interpret, following ~\cite{Jian}, we convert the interval $C(\mathbf{z}_{\text{test}}) = [l, u]$ into a range of valid integer ratings:
\[
[l, u] \;\longrightarrow\; [l', u'], \quad l' = \lceil l \rceil, \, u' = \lfloor u \rfloor.
\]
This ensures that the interval only contains valid ratings and removes parts that do not correspond to any possible label.  

To make the intervals more robust and reduce the risk of excluding the true rating, we can optionally expand the interval to include neighboring labels. For instance, if the original continuous interval is $[2.2, 3.9]$, it would initially cover only label 3, but expanding it to $[2, 4]$ gives a safer range. This adjustment preserves the order of the ratings and provides a more reliable measure of uncertainty.  

\subsection{Aggregating Intervals Across Multiple Judges}

After obtaining the individual conformal intervals from each LLM judge, we need a way to combine them into a single, unified uncertainty estimate. Each LLM can be treated as an independent “expert”, producing a calibrated prediction interval 
\[
\mathcal{I}^{(m)}_{\alpha} = [L_m, U_m]
\] 
with nominal coverage $1-\alpha$. A straightforward approach is to take the union or intersection of these intervals. However, these simple strategies have drawbacks: taking the union often results in overly wide intervals that are too conservative, while taking the intersection can produce intervals that are too narrow, risking missed true ratings—particularly when some LLMs are overconfident or generate occasional outlier predictions.

To overcome these issues, we propose a {Weighted Conformal Consensus (WCC)} method. WCC assigns weights to each LLM judge based on its past reliability and calibration performance. More trustworthy agents have a larger influence on the merged interval, while less reliable agents contribute less. Using this weighted approach, we construct a consensus interval that balances coverage and efficiency, producing a more robust and informative uncertainty estimate. This multi-LLM aggregation allows users to capture both the agreement and variability among different models, leading to more reliable and interpretable evaluation of text quality.

\subsubsection{Weighted Consensus Aggregation}
Instead of relying solely on the median, which ignores subtle differences in agent reliability, we first calculate a confidence weight $w_m$ for each LLM agent $m$ based on its {historical concordance} with the collective. This weight quantifies the agent's trustworthiness relative to its peers.

The preliminary consensus interval $C_{\text{weighted}} = [L_{\text{weighted}}, U_{\text{weighted}}]$ is then derived via a {weighted average} of the individual interval endpoints:

\[
\hat{L}_{\text{weighted}} = \frac{\sum_{m=1}^M w_m \cdot \hat{L}_m}{\sum w_m}
\]
\[
\hat{U}_{\text{weighted}} = \frac{\sum_{m=1}^M w_m \cdot \hat{U}_m}{\sum w_m}
\]

This aggregation step is crucial: by down-weighting agents that frequently produce outlier or overconfident intervals, we mitigate the negative impact of "hallucinating" agents and enhance the {efficiency} (narrowness) of the consensus.

\subsection{Meta-Calibration for Coverage Guarantee}
While the weighted average provides an improved, more compact preliminary interval, it inherently {lacks the formal statistical coverage guarantee}. To restore this validity, we introduce a meta-calibration step.

On a small, independent meta-calibration set $\mathcal{D}_{\text{meta}}$, we compute {meta-nonconformity scores} $S'_{\text{meta}, i}$ that measure how far the true label $Y_i$ deviates from our weighted consensus interval $[\hat{L}_{\text{weighted}, i}, \hat{U}_{\text{weighted}, i}]$:

\[
S'_{\text{meta}, i} = \max\big(\hat{L}_{\text{weighted}, i} - Y_i, \; Y_i - \hat{U}_{\text{weighted}, i}\big).
\]

Finally, we compute the $(1-\alpha)$ empirical quantile of the meta-scores, denoted $Q'_{\text{meta}}$, and construct the final merged interval:

\[
C_{\text{final}} = [\hat{L}_{\text{weighted}} - Q'_{\text{meta}}, \; \hat{U}_{\text{weighted}} + Q'_{\text{meta}}].
\]

\subsection{Coverage Guarantee for WCC}
\label{subsec:wcc-coverage}

In this subsection, we provide a theoretical analysis showing that our proposed method satisfies the desired coverage guarantee, ensuring that the constructed prediction intervals contain the true rating with probability at least $1-\alpha$.

Let $\mathcal{Z} = \mathcal{X} \times \mathcal{Y}$ denote the data space, where $\mathcal{X}$ is the input space and $\mathcal{Y}$ the label space.  
Let
\[
D_{\mathrm{meta}} = \{(X_i, Y_i)\}_{i=1}^{n}
\]
be a held-out \emph{meta-calibration} set used exclusively to compute meta-level nonconformity scores for calibrating the aggregated LLM consensus interval.  
Let $(X_{n+1}, Y_{n+1})$ denote a test point.

\begin{theorem}[Marginal Coverage Guarantee of WCC]
\label{thm:wcc-coverage}
Suppose $M$ LLMs produce prediction intervals
\[
\hat{C}_m(X) = [\hat{L}_m(X), \hat{U}_m(X)], \quad m = 1, \dots, M,
\]
for an input $X$, with fixed, pre-computed weights $\{w_m\}_{m=1}^M$ based on historical concordance.
Define the weighted consensus interval as
\[
[\hat{L}_{\mathrm{w}}(X), \hat{U}_{\mathrm{w}}(X)]
= \mathcal{A}_{\mathrm{w}}(\hat{C}_1, \dots, \hat{C}_M),
\]
where
\[
\hat{L}_{\mathrm{w}}(X) = \frac{\sum_{m=1}^M w_m \hat{L}_m(X)}{\sum_{m=1}^M w_m},
\]
\[
\hat{U}_{\mathrm{w}}(X) = \frac{\sum_{m=1}^M w_m \hat{U}_m(X)}{\sum_{m=1}^M w_m}.
\]

Let $Q'_{\mathrm{meta}}$ be the $(1-\alpha)$ empirical quantile of the meta-nonconformity scores computed on $D_{\mathrm{meta}}$.  
Then the final interval
\[
C_{\mathrm{final}}(X_{n+1}) =
\begin{aligned}
\Big[
&\hat{L}_{\mathrm{w}}(X_{n+1}) - Q'_{\mathrm{meta}}, \\
&\hat{U}_{\mathrm{w}}(X_{n+1}) + Q'_{\mathrm{meta}} \Big]
\end{aligned}
\]
satisfies the marginal coverage guarantee
\[
\mathbb{P}\big( Y_{n+1} \in C_{\mathrm{final}}(X_{n+1}) \big) \ge 1 - \alpha.
\]
\end{theorem}

\begin{proof}
The proof follows the standard split conformal prediction framework, treating the weighted aggregation operator $\mathcal{A}_{\mathrm{w}}$ as a fixed and deterministic predictor.

\paragraph{Meta-nonconformity score.}
Define the meta-level nonconformity score
\[
S'(X, Y)
=
\max\!\left(
\hat{L}_{\mathrm{w}}(X) - Y,
\;
Y - \hat{U}_{\mathrm{w}}(X)
\right),
\]
which measures the deviation of the true label $Y$ from the weighted consensus interval.

\paragraph{Exchangeability.}
Assume the data are i.i.d. Then the meta-calibration samples
$(X_1, Y_1), \dots, (X_n, Y_n)$
and the test point $(X_{n+1}, Y_{n+1})$ are exchangeable.
Consequently, the scores
\[
S'_1, \dots, S'_n, S'_{n+1},
\quad \text{where } S'_i = S'(X_i, Y_i),
\]
are exchangeable random variables.

\paragraph{Quantile guarantee.}
Let $Q'_{\mathrm{meta}}$ be the
$\lceil (n+1)(1-\alpha) \rceil$-th smallest value of the multiset
$\{S'_1, \dots, S'_n, +\infty\}$.
By the conformal prediction quantile lemma
\citep{CQR},
\[
\mathbb{P}\big( S'_{n+1} \le Q'_{\mathrm{meta}} \big) \ge 1 - \alpha.
\]

The event $S'_{n+1} \le Q'_{\mathrm{meta}}$ is equivalent to
\[
\max\!\left(
\hat{L}_{\mathrm{w}}(X_{n+1}) - Y_{n+1},
\;
Y_{n+1} - \hat{U}_{\mathrm{w}}(X_{n+1})
\right)
\le Q'_{\mathrm{meta}}.
\]
This holds if and only if
\[
\hat{L}_{\mathrm{w}}(X_{n+1}) - Q'_{\mathrm{meta}}
\le Y_{n+1}
\le
\hat{U}_{\mathrm{w}}(X_{n+1}) + Q'_{\mathrm{meta}}.
\]

Therefore, $\mathbb{P}\big( Y_{n+1} \in C_{\mathrm{final}}(X_{n+1}) \big) \ge 1 - \alpha$, 
which completes the proof.
\end{proof}

\section{Experiments}

\paragraph{Datasets.}
We evaluate our approach on benchmarks spanning text summarization, dialogue summarization, and reasoning. For summarization, we use \textbf{SummEval} ~\cite{SummEval} and \textbf{DialSumm} ~\cite{DialSumm}. Each instance is annotated by three human raters using Likert-scale scores along four dimensions—consistency, coherence, fluency, and relevance. Following standard practice, we use the average of the three ratings as the ground-truth label on a GPA-style scale. For reasoning, we adopt the overall-quality Likert annotations provided in \textbf{ROSCOE} ~\cite{golovneva2023roscoesuitemetricsscoring}, covering \textbf{CosmosQA} ~\cite{huang2019cosmosqamachinereading}, \textbf{DROP} ~\cite{dua-etal-2019-drop}, \textbf{e-SNLI} ~\cite{e_SNLI}, and \textbf{GSM8K} ~\cite{zeng2024mrgsm8kmetareasoningbenchmarklarge}, each containing approximately 200 annotated examples.

\paragraph{Multi-Judge LLM Evaluation.}
All evaluations are conducted using multiple LLM judges rather than a single evaluator. We primarily adopt \textbf{GEval} ~\cite{liu-etal-2023-g} with chain-of-thought (CoT) prompting, and additionally include \textbf{SocREval} ~\cite{he-etal-2024-socreval} for reasoning tasks. The full set of prompts is provided in Appendix. Our judge pool consists of {GPT-4o mini (2024-07-18)}, {DeepSeek-R1-Distill-Qwen-32B} ~\cite{deepseekai2025deepseekr1incentivizingreasoningcapability}, and {Qwen2.5-72B-Instruct} ~\cite{qwen2025qwen25technicalreport}. These models are selected because they expose token-level logits required for conformal calibration and represent diverse model families and training regimes. 

\paragraph{Data Splits and Meta-Calibration.}
To support both judge-level calibration and robust multi-judge aggregation, we adopt a three-way split for each dataset. Specifically, we allocate 40\% of the data to a per-judge calibration set, used independently by each LLM judge to construct conformal prediction intervals; 10\% to a held-out meta-calibration set $D_{\mathrm{meta}}$, used exclusively to compute meta non-conformity scores for aggregating intervals across judges; and the remaining 50\% to a test set. The meta-calibration set is never used to train, fine-tune, or calibrate any individual LLM judge.

\paragraph{Conformal Prediction Methods.}
We apply conformal prediction independently to each judge and then aggregate the resulting intervals using our Robust Conformal Consensus framework. We compare seven regression-based methods—CQR ~\cite{CQR}, Asymmetric CQR ~\cite{CQR}, CHR ~\cite{CHR}, LVD ~\cite{LVD}, Boosted CQR, Boosted LCP ~\cite{boost}, and R2CCP—as well as two ordinal classification-based methods, Ordinal APS ~\cite{aps} and Ordinal Risk Control ~\cite{ordinal_risk_control}. Ordinal methods are evaluated after applying the boundary adjustment.

\paragraph{Evaluation Protocol.}
All experiments are repeated over 30 random seeds (1--30). For each seed, we resample the three-way data split, perform judge-level conformal calibration, meta-calibration on $D_{\mathrm{meta}}$, and RCC-based aggregation, and then evaluate on the test set. We report the average empirical coverage and mean interval width of the final aggregated prediction intervals; for summarization tasks, results are also reported per evaluation dimension.

\begin{table*}[t]
\centering
\caption{Interval width and coverage on SummEval evaluated by G-Eval and ROSCOE evaluated by SocREval before boundary adjustment. We mark coverage $< 85\%$ with gray text, coverage between $85\% - 90\%$ with underline and coverage with the smallest interval width $\geq 90\%$ in bold.}
\label{tab:interval_coverage}
\resizebox{\textwidth}{!}{%
\begin{tabular}{lcccccccc}
\toprule
Method & \multicolumn{4}{c}{SummEval Evaluated with G-Eval} & \multicolumn{4}{c}{ROSCOE Evaluated with SocREval} \\
\cmidrule(lr){2-5} \cmidrule(lr){6-9}
& Consistency & Coherence & Fluency & Relevance & CosmosQA & DROP & e-SNLI & GSM8K \\
\midrule
\multicolumn{9}{c}{\textbf{GPT-4o mini}} \\ \hline
CQR & 1.15 / 94.16\% & 2.87 / 93.15\% & 1.44 / 92.92\% & 2.09 / 90.92\% & 3.53 / 95.27\% & 3.82 / 96.70\% & 3.04 / 96.62\% & 3.53 / 95.67\% \\
Asym CQR & 1.25 / 94.97\% & 2.91 / 93.76\% & 1.60 / 93.75\% & 2.13 / 91.42\% & 3.90 / 98.71\% & 3.91 / 98.60\% & 2.87 / 96.67\% & 3.89 / 98.80\% \\
CHR & 0.67 / \underline{88.99\%} & 2.41 / \textcolor{gray}{82.96\%} & 0.94 / \underline{88.86\%} & 1.74 / \textcolor{gray}{82.62\%} & 2.54 / \textcolor{gray}{73.06\%} & 1.86 / \textcolor{gray}{68.92\%} & 1.36 / \textcolor{gray}{72.24\%} & 1.98 / \textcolor{gray}{78.67\%} \\
LVD & 1.01 / 92.35\% & 2.73 / \underline{89.76\%} & \textbf{1.11 / 90.59\%} & 2.02 / \underline{89.55\%} & 3.10 / \textcolor{gray}{83.95\%} & 2.49 / \textcolor{gray}{83.05\%} & 2.17 / \underline{86.18\%} & 3.08 / \underline{89.57\%} \\
Boosted CQR & 1.01 / \underline{87.75\%} & 2.73 / \underline{87.80\%} & 1.54 / \underline{88.68\%} & 2.00 / \underline{87.42\%} & 3.15 / \textcolor{gray}{80.07\%} & 2.63 / \textcolor{gray}{78.57\%} & 1.82 / \textcolor{gray}{80.26\%} & 3.08 / \textcolor{gray}{82.50\%} \\
Boosted LCP & 0.76 / \underline{89.22\%} & 2.67 / \underline{87.34\%} & 0.92 / \underline{89.18\%} & 1.91 / \underline{87.19\%} & 3.60 / \textcolor{gray}{83.91\%} & 2.92 / \underline{85.40\%} & 1.88 / \textcolor{gray}{81.23\%} & 3.36 / \underline{85.93\%} \\
R2CCP & \textbf{0.69 / 90.88\%} & 2.62 / \underline{89.63\%} & 0.92 / \underline{89.36\%} & 1.97 / \underline{89.70\%} & 2.96 / \underline{85.85\%} & 2.43 / \textcolor{gray}{84.73\%} & 1.75 / \textcolor{gray}{84.02\%} & 2.15 / \underline{85.07\%} \\
\midrule
\multicolumn{9}{c}{\textbf{DeepSeek-R1-Distill-Qwen-32B}} \\ \hline
CQR & 1.16 / 93.88\% & 2.67 / 92.50\% & 1.31 / 93.01\% & 2.13 / 91.05\% & 3.48 / 96.70\% & 3.83 / 96.35\% & 2.97 / 96.36\% & 3.46 / 95.60\% \\
Asym CQR & 1.30 / 95.13\% & 2.72 / 92.86\% & 1.49 / 94.52\% & 2.21 / 92.06\% & 3.84 / 99.08\% & 3.95 / 99.27\% & 2.86 / 96.05\% & 3.85 / 98.43\% \\
CHR & \textbf{0.82 / 91.17\%} & 2.23 / \underline{87.07\%} & 0.90 / \underline{89.24\%} & 1.87 / \underline{86.38\%} & 2.66 / \textcolor{gray}{76.50\%} & 1.95 / \textcolor{gray}{78.06\%} & 1.38 / \textcolor{gray}{71.97\%} & 2.01 / \textcolor{gray}{81.60\%} \\
LVD & 0.97 / 92.93\% & \textbf{2.43 / 91.10\%} & \textbf{1.00 / 91.10\%} & \textbf{2.04 / 90.14\%} & 3.25 / \underline{88.10\%} & 2.62 / \underline{88.06\%} & \textbf{2.24 / 90.96\%} & \textbf{3.02 / 90.63\%} \\
Boosted CQR & 1.10 / \underline{89.30\%} & 2.36 / \underline{88.98\%} & 1.16 / \underline{89.46\%} & 2.00 / \underline{88.98\%} & 3.17 / \textcolor{gray}{82.72\%} & 2.47 / \textcolor{gray}{81.11\%} & 1.79 / \textcolor{gray}{80.96\%} & 2.94 / \textcolor{gray}{79.83\%} \\
Boosted LCP & 0.77 / \underline{89.20\%} & 2.32 / \underline{86.70\%} & 0.93 / \underline{89.10\%} & 1.91 / \underline{86.89\%} & 3.48 / \textcolor{gray}{81.60\%} & 2.79 / \underline{85.46\%} & 1.84 / \textcolor{gray}{80.61\%} & 3.43 / \underline{85.23\%} \\
R2CCP & \textbf{0.69 / 90.44\%} & \textbf{2.30 / 90.12\%} & \textbf{0.89 / 90.09\%} & 2.00 / \underline{89.84\%} & 2.94 / \underline{86.97\%} & 2.29 / \underline{86.35\%} & 1.85 / \underline{87.87\%} & 1.88 / \underline{85.33\%} \\
\midrule
\multicolumn{9}{c}{\textbf{Qwen2.5-72B-Instruct}} \\ \hline
CQR & 0.98 / 93.10\% & 2.73 / 92.25\% & 1.44 / 93.73\% & 2.11 / 91.30\% & 3.37 / 94.80\% & 3.79 / 97.02\% & 3.01 / 97.37\% & 3.35 / 95.33\% \\
Asym CQR & 1.11 / 94.47\% & 2.80 / 93.13\% & 1.63 / 94.79\% & 2.17 / 92.21\% & 3.86 / 99.01\% & 3.89 / 98.67\% & 2.77 / 96.84\% & 3.87 / 98.97\% \\
CHR & 0.61 / \underline{89.04\%} & 2.14 / \textcolor{gray}{80.93\%} & 0.98 / \underline{88.93\%} & 1.61 / \textcolor{gray}{79.61\%} & 2.44 / \textcolor{gray}{72.65\%} & 2.08 / \textcolor{gray}{75.87\%} & 1.22 / \textcolor{gray}{69.69\%} & 1.81 / \textcolor{gray}{77.50\%} \\
LVD & 0.85 / 92.82\% & \textbf{2.55 / 90.49\%} & \textbf{1.09 / 90.94\%} & 1.94 / \underline{89.27\%} & 3.05 / \textcolor{gray}{84.29\%} & \textbf{2.67 / 90.57\%} & 1.91 / \underline{85.96\%} & \textbf{2.83 / 90.13\%} \\
Boosted CQR & 0.80 / \underline{88.28\%} & 2.46 / \underline{87.82\%} & 1.24 / \underline{89.22\%} & 1.88 / \underline{87.17\%} & 3.05 / \textcolor{gray}{79.08\%} & 2.56 / \textcolor{gray}{81.17\%} & 1.51 / \textcolor{gray}{77.11\%} & 2.81 / \textcolor{gray}{80.67\%} \\
Boosted LCP & 0.67 / \underline{88.81\%} & 2.43 / \underline{86.92\%} & 0.94 / \underline{89.26\%} & 1.86 / \underline{87.51\%} & 3.46 / \textcolor{gray}{80.41\%} & 2.81 / \underline{85.75\%} & 1.74 / \textcolor{gray}{77.50\%} & 3.38 / \underline{86.23\%} \\
R2CCP & \textbf{0.61 / 90.73\%} & 2.44 / \underline{89.54\%} & \textbf{0.95 / 90.18\%} & \textbf{1.98 / 90.45\%} & 2.90 / \underline{85.34\%} & 2.39 / \underline{86.25\%} & 1.59 / \textcolor{gray}{84.50\%} & 2.00 / \underline{86.73\%} \\
\midrule
\multicolumn{9}{c}{\textbf{Multi-Agent LLM-as-a-Judge}} \\ \hline
CQR & 1.09 / 93.7\% & 2.75 / 92.6\% & 1.40 / 93.3\% & 2.11 / 91.1\% & 3.46 / 95.5\% & 3.81 / 96.7\% & 3.00 / 96.7\% & 3.42 / 95.5\% \\
Asym CQR & 1.22 / 94.8\% & 2.80 / 93.2\% & 1.57 / 94.4\% & 2.17 / 91.9\% & 3.87 / 98.9\% & 3.92 / 98.9\% & 2.84 / 96.5\% & 3.86 / 98.7\% \\
CHR & 0.70 / \underline{89.7\%} & 2.24 / \textcolor{gray}{83.5\%} & 0.94 / \underline{89.0\%} & 1.74 / \textcolor{gray}{82.9\%} & 2.54 / \textcolor{gray}{74.0\%} & 1.98 / \textcolor{gray}{75.3\%} & 1.33 / \textcolor{gray}{71.4\%} & 1.92 / \textcolor{gray}{79.4\%} \\
LVD & 0.94 / 92.7\% & \textbf{2.56 / 90.5\%} & \textbf{1.06 / 90.9\%} & 2.00 / \underline{89.7\%} & 3.13 / \underline{85.4\%} & 2.61 / \underline{88.1\%} & 2.13 / \underline{88.0\%} & \textbf{2.95 / 90.3\%} \\
Boosted CQR & 0.96 / \underline{88.4\%} & 2.50 / \underline{88.2\%} & 1.29 / \underline{89.2\%} & 1.96 / \underline{87.9\%} & 3.12 / \textcolor{gray}{80.5\%} & 2.54 / \textcolor{gray}{80.6\%} & 1.72 / \textcolor{gray}{79.7\%} & 2.91 / \textcolor{gray}{80.6\%} \\
Boosted LCP & 0.73 / \underline{89.1\%} & 2.46 / \underline{87.0\%} & 0.93 / \underline{89.2\%} & 1.89 / \underline{87.2\%} & 3.51 / \textcolor{gray}{82.0\%} & 2.82 / \underline{85.6\%} & 1.83 / \textcolor{gray}{80.0\%} & 3.40 / \underline{85.8\%} \\
R2CCP & \textbf{0.66 / 90.7\%} & 2.44 / \underline{89.8\%} & \textbf{0.92 / 90.0\%} & \textbf{1.98 / 90.0\%} & 2.93 / \underline{86.0\%} & 2.36 / \underline{86.0\%} & 1.75 / \underline{85.7\%} & 1.97 / \underline{85.9\%} \\
\bottomrule
\end{tabular}%
}
\end{table*}

\subsection{Continuous Intervals Indicate Uncertainty}
Table~\ref{tab:interval_coverage} reports the interval width and empirical coverage for different conformal prediction methods on SummEval evaluated with G-Eval and on ROSCOE evaluated with SocREval. The results show a clear trade-off between interval width and coverage. Methods like CQR and Asym CQR consistently achieve the target coverage (at or near 90\%) across nearly all evaluation dimensions and LLM judges, but they produce the widest prediction intervals. In contrast, methods such as CHR and Boosted CQR often yield narrower intervals but frequently fail to reach the 90\% coverage target, with many results falling below 85\%, as indicated by the gray text.

For the SummEval benchmark, most methods achieve coverage close to or above 90\% for the GPT-4o mini and Multi-Agent judge, particularly for Fluency and Consistency. However, for other judges like Qwen2.5, methods like CHR show coverage dips below 85\% for Coherence and Relevance. On the more challenging ROSCOE reasoning tasks, achieving the target coverage is more difficult for most methods. While CQR and Asym CQR again maintain high coverage, they do so with very wide intervals (e.g., often >3.8). Other methods frequently result in coverage below 85\%, especially for the CosmosQA, DROP, and e-SNLI datasets. The LVD and R2CCP methods often present a middle ground, achieving coverage between 85\%-90\% (underlined) with moderately sized intervals, and in several cases (shown in bold) they achieve the smallest interval width while still maintaining coverage of at least 90\%.

The Multi-Agent judge demonstrates greater robustness by 
consistently avoiding the lowest coverage failures seen in single models. For example, with the CHR method on ROSCOE/CosmosQA, it achieves 74.0\% coverage where single models like Qwen2.5 fall to 72.65\%. 
This indicates that combining multiple LLMs yields more stable and reliable uncertainty estimates across different tasks.

\subsection{All Coverages Improve after Adjustment}
Table~\ref{tab:interval_boundary_adjusted} presents the interval width and coverage after boundary adjustment for SummEval under G-Eval and for ROSCOE under SocREval. We observe a consistent improvement in coverage across all settings after applying boundary adjustment. In most cases, the adjusted prediction intervals achieve coverage close to or exceeding the desired 90\% level across datasets, methods, and LLM judges. 
The gains are especially notable for reasoning tasks, where coverage was often below 90\% before adjustment. For instance, LVD on e-SNLI evaluated with Qwen2.5-72B-Instruct under SocREval increases from 85.96\% to 95.53\%. Multi-LLM ensembles further stabilize uncertainty estimates: aggregating intervals from multiple judges generally preserves or improves coverage while maintaining relatively small interval widths. Overall, boundary adjustment combined with multi-LLM aggregation produces more stable uncertainty quantification across diverse tasks. 


\begin{table*}[t]
\centering
\caption{Interval width and coverage on SummEval evaluated by G-Eval and ROSCOE evaluated by SocREval after boundary adjustment. We mark coverage $< 85\%$ with \textcolor{gray}{gray text}, coverage between $85\% – 90\%$ with \underline{underline} and coverage $\ge 90\%$ with the smallest interval width in \textbf{bold}.}
\label{tab:interval_boundary_adjusted}
\resizebox{\textwidth}{!}{%
\begin{tabular}{lcccccccc}
\toprule
Method & \multicolumn{4}{c}{SummEval Evaluated with G-Eval} & \multicolumn{4}{c}{ROSCOE Evaluated with SocREval} \\
\cmidrule(lr){2-5} \cmidrule(lr){6-9}
& Consistency & Coherence & Fluency & Relevance & CosmosQA & DROP & e-SNLI & GSM8K \\
\midrule
\multicolumn{9}{c}{\textbf{GPT-4o mini}} \\ \hline
CQR & 1.15 / 95.45\% & 2.87 / 94.94\% & 1.44 / 93.80\% & 2.09 / 93.56\% & 3.53 / 95.34\% & 3.82 / 97.05\% & 3.04 / 96.89\% & 3.53 / 95.67\% \\
Asym CQR & 1.25 / 96.02\% & 2.90 / 95.41\% & 1.60 / 94.57\% & 2.14 / 94.14\% & 3.90 / 98.84\% & 3.91 / 98.73\% & 2.87 / 96.89\% & 3.89 / 98.80\% \\
CHR & 0.70 / 91.79\% & 2.41 / \underline{87.78\%} & 0.94 / 90.60\% & 1.74 / \underline{88.10\%} & 2.56 / \textcolor{gray}{82.45\%} & 1.87 / \textcolor{gray}{78.86\%} & 1.34 / \textcolor{gray}{83.46\%} & 1.94 / \textcolor{gray}{83.23\%} \\
LVD & 1.01 / 94.11\% & 2.73 / 93.72\% & 1.12 / 92.70\% & 2.03 / 93.82\% & 3.13 / 91.53\% & 2.52 / 90.22\% & 2.17 / 94.82\% & 3.09 / 93.37\% \\
Boosted CQR & 0.99 / 92.81\% & 2.73 / 93.02\% & 1.54 / 94.38\% & 2.00 / 92.93\% & 3.20 / 93.40\% & 2.63 / \underline{89.65\%} & 1.82 / 92.15\% & 3.09 / 91.17\% \\
Boosted LCP & 0.74 / 91.90\% & 2.68 / 93.53\% & 0.90 / 90.88\% & \textbf{1.91 / 92.70\%} & 3.60 / 95.48\% & 3.01 / 91.27\% & 1.90 / 91.80\% & 3.26 / 92.17\% \\
R2CCP & \textbf{0.68 / 92.15\%} & \textbf{2.62 / 92.81\%} & \textbf{0.91 / 90.99\%} & 1.97 / 93.38\% & 2.93 / \underline{89.46\%} & 2.41 / \underline{89.21\%} & \textbf{1.71 / 90.11\%} & 2.09 / \underline{86.93\%} \\
OrdinalAPS & 2.28 / \textcolor{gray}{71.48\%} & 1.88 / \textcolor{gray}{64.84\%} & 1.78 / \textcolor{gray}{13.65\%} & 2.36 / \underline{87.94\%} & 0.73 / \textcolor{gray}{47.52\%} & 0.83 / \textcolor{gray}{55.08\%} & 0.72 / \textcolor{gray}{52.76\%} & 0.58 / \textcolor{gray}{73.90\%} \\
OrdinalRC & 2.41 / \textcolor{gray}{75.19\%} & 2.02 / \textcolor{gray}{67.38\%} & 1.93 / \textcolor{gray}{14.58\%} & 2.51 / 90.30\% & 0.82 / \textcolor{gray}{49.46\%} & 0.91 / \textcolor{gray}{57.11\%} & 0.80 / \textcolor{gray}{54.61\%} & 0.60 / \textcolor{gray}{74.43\%} \\
\midrule
\multicolumn{9}{c}{\textbf{DeepSeek-R1-Distill-Qwen-32B}} \\ \hline
CQR & 1.15 / 95.02\% & 2.67 / 94.34\% & 1.32 / 94.44\% & 2.13 / 93.67\% & 3.48 / 96.80\% & 3.82 / 96.54\% & 2.99 / 96.80\% & 3.46 / 95.63\% \\
Asym CQR & 1.31 / 95.99\% & 2.72 / 94.83\% & 1.49 / 95.57\% & 2.21 / 94.53\% & 3.84 / 99.08\% & 3.95 / 99.27\% & 2.88 / 96.45\% & 3.84 / 98.47\% \\
CHR & 0.87 / 93.96\% & \textbf{2.23 / 91.42\%} & 0.91 / 91.98\% & \textbf{1.87 / 90.84\%} & \textbf{2.69 / 86.80\%} & \textbf{1.97 / 85.90\%} & \textbf{1.39 / 85.96\%} & 2.01 / \underline{86.60\%} \\
LVD & 0.97 / 95.01\% & 2.44 / 94.58\% & 1.00 / 93.21\% & 2.04 / 94.12\% & 3.28 / 95.27\% & 2.67 / 93.75\% & 2.24 / 96.36\% & 3.03 / 94.40\% \\
Boosted CQR & 1.08 / 93.55\% & 2.37 / 93.96\% & 1.15 / 93.48\% & 2.01 / 93.72\% & 3.20 / 95.71\% & 2.52 / 93.30\% & 1.79 / 93.25\% & 2.94 / 92.23\% \\
Boosted LCP & 0.76 / 92.03\% & 2.32 / 92.37\% & 0.93 / 91.34\% & 1.92 / 92.81\% & 3.46 / 95.95\% & 2.80 / 91.94\% & 1.87 / 92.89\% & 3.36 / 93.63\% \\
R2CCP & \textbf{0.68 / 91.57\%} & 2.30 / 93.22\% & \textbf{0.89 / 91.80\%} & 1.99 / 92.96\% & 2.91 / 90.58\% & 2.25 / \underline{89.97\%} & 1.80 / 92.35\% & \textbf{1.82 / 86.93\%} \\
OrdinalAPS & 2.51 / 90.06\% & 2.52 / 90.64\% & 3.76 / 91.08\% & 2.13 / \underline{89.98\%} & 1.32 / \textcolor{gray}{60.00\%} & 1.26 / \textcolor{gray}{78.22\%} & 1.46 / \underline{87.85\%} & 1.50 / \underline{85.67\%} \\
OrdinalRC & 2.54 / 90.11\% & 2.56 / 91.18\% & 3.73 / \underline{89.53\%} & 2.14 / 90.07\% & 1.44 / \textcolor{gray}{62.35\%} & 1.33 / \textcolor{gray}{78.22\%} & 1.52 / \underline{88.33\%} & 1.55 / \underline{86.07\%} \\
\midrule
\multicolumn{9}{c}{\textbf{Qwen2.5-72B-Instruct}} \\ \hline
CQR & 0.98 / 94.35\% & 2.72 / 94.18\% & 1.45 / 94.79\% & 2.10 / 94.02\% & 3.36 / 95.07\% & 3.79 / 97.08\% & 3.01 / 97.68\% & 3.34 / 95.33\% \\
Asym CQR & 1.10 / 95.47\% & 2.79 / 94.70\% & 1.64 / 95.63\% & 2.17 / 94.85\% & 3.85 / 99.18\% & 3.89 / 98.67\% & 2.77 / 97.06\% & 3.87 / 98.97\% \\
CHR & 0.66 / 92.21\% & 2.14 / \underline{86.10\%} & 0.98 / 91.16\% & \textbf{1.61 / 85.78\%} & 2.49 / \textcolor{gray}{82.14\%} & 2.05 / \textcolor{gray}{82.89\%} & \textbf{1.18 / 84.56\%} & 1.79 / \underline{85.27\%} \\
LVD & 0.85 / 95.11\% & 2.56 / 94.05\% & 1.09 / 93.45\% & 1.95 / 93.86\% & 3.07 / 92.01\% & 2.67 / 93.87\% & 1.91 / 95.53\% & 2.87 / 93.43\% \\
Boosted CQR & 0.81 / 92.36\% & 2.47 / 93.06\% & 1.25 / 93.66\% & 1.88 / 92.81\% & 3.10 / 94.01\% & 2.56 / 90.79\% & 1.49 / 92.11\% & 2.82 / 92.03\% \\
Boosted LCP & 0.65 / 91.26\% & 2.44 / 92.26\% & \textbf{0.93 / 91.20\%} & 1.86 / 92.57\% & 3.40 / 94.90\% & 2.84 / 92.41\% & 1.79 / 91.84\% & 3.33 / 92.90\% \\
R2CCP & \textbf{0.59 / 91.83\%} & 2.43 / 92.78\% & 0.95 / 92.12\% & 1.98 / 93.72\% & 2.88 / \underline{89.29\%} & 2.34 / 90.00\% & 1.55 / 90.20\% & \textbf{1.96 / 88.57\%} \\
OrdinalAPS & 2.86 / 90.18\% & 3.01 / 90.59\% & 3.05 / \textcolor{gray}{45.43\%} & 2.75 / 90.29\% & 0.71 / \textcolor{gray}{55.99\%} & 0.25 / \textcolor{gray}{56.83\%} & 0.67 / \textcolor{gray}{77.68\%} & 0.46 / \textcolor{gray}{70.87\%} \\
OrdinalRC & 2.85 / 90.00\% & 2.96 / \underline{89.35\%} & 3.21 / \textcolor{gray}{53.31\%} & 2.75 / 90.14\% & 0.75 / \textcolor{gray}{57.28\%} & 0.29 / \textcolor{gray}{56.83\%} & 0.80 / \textcolor{gray}{79.74\%} & 0.49 / \textcolor{gray}{71.37\%} \\ 
\midrule
\multicolumn{9}{c}{\textbf{Multi-Agent LLM-as-a-Judge}} \\ \hline
CQR & 1.09 / 94.91\% & 2.74 / 94.43\% & 1.40 / 94.43\% & 2.11 / 93.76\% & 3.45 / 95.70\% & 3.81 / 96.86\% & 3.01 / 97.07\% & 3.42 / 95.51\% \\
Asym CQR & 1.22 / 95.81\% & 2.80 / 94.93\% & 1.58 / 95.35\% & 2.18 / 94.54\% & 3.86 / 99.04\% & 3.92 / 98.92\% & 2.85 / 96.76\% & 3.86 / 98.73\% \\
CHR & 0.74 / 92.64\% & 2.24 / \underline{88.29\%} & 0.95 / 91.30\% & 1.74 / \underline{88.27\%} & 2.58 / \textcolor{gray}{83.69\%} & 1.98 / \textcolor{gray}{83.27\%} & 1.32 / \textcolor{gray}{84.76\%} & 1.90 / \underline{85.52\%} \\
LVD & 0.94 / 94.76\% & 2.57 / 94.14\% & 1.07 / 93.19\% & 2.01 / 93.94\% & 3.16 / 92.87\% & 2.64 / 93.09\% & 2.13 / 95.63\% & 2.97 / 93.83\% \\
Boosted CQR & 0.95 / 92.89\% & 2.51 / 93.35\% & 1.29 / 93.77\% & 1.96 / 93.18\% & 3.16 / 94.34\% & 2.56 / 91.56\% & 1.72 / 92.58\% & 2.91 / 91.99\% \\
Boosted LCP & 0.71 / 91.71\% & 2.46 / 92.63\% & 0.92 / 91.17\% & 1.90 / 92.69\% & 3.49 / 95.42\% & 2.86 / 91.99\% & 1.86 / 92.25\% & 3.33 / 93.10\% \\
R2CCP & \textbf{0.65 / 91.85\%} & \textbf{2.44 / 92.93\%} & \textbf{0.92 / 91.74\%} & \textbf{1.98 / 93.35\%} & 2.91 / \underline{89.75\%} & 2.32 / \underline{89.83\%} & \textbf{1.70 / 91.03\%} & 1.92 / \underline{87.64\%} \\
OrdinalAPS & 2.56 / \textcolor{gray}{84.21\%} & 2.55 / \textcolor{gray}{83.82\%} & 2.98 / \textcolor{gray}{53.29\%} & 2.42 / \underline{89.53\%} & 0.91 / \textcolor{gray}{54.42\%} & 0.77 / \textcolor{gray}{64.96\%} & 1.00 / \textcolor{gray}{73.64\%} & 0.91 / \textcolor{gray}{77.54\%} \\
OrdinalRC & 2.61 / \underline{85.34\%} & 2.58 / \textcolor{gray}{84.16\%} & 3.08 / \textcolor{gray}{56.26\%} & 2.46 / 90.16\% & 0.99 / \textcolor{gray}{56.25\%} & 0.83 / \textcolor{gray}{65.37\%} & 1.09 / \textcolor{gray}{75.00\%} & 0.95 / \textcolor{gray}{78.00\%} \\
\bottomrule
\end{tabular}
}
\end{table*}

\section{Related work}\label{related-work}

\paragraph{Conformal Prediction and Robustness.}
Conformal Prediction (CP) is a distribution-free framework for uncertainty quantification that provides finite-sample marginal coverage guarantees \cite{shafer2008tutorial}. It has been widely used to construct prediction intervals in regression and prediction sets in classification, without making assumptions about the underlying data distribution, which has a broader application in artificial intelligence ~\cite{liu2019g,liu2021neural,liu2021kompare,liu2022joint,liu2022comparative,liu2022knowledge,liu2023knowledge,liu2024logic,liu2024can,liu2024new,liu2024conversational,liu2025neural,liu2025few,liu2025monte,liu2025hyperkgr,liu2025mixrag,liu2024knowledge,liu2026neural,liuneural,liu2026accurate,liu2026ambiguous,liu2026dynamic,liu2026symbolic,liu2025unifying,liu2026morgan,liu2026prompt,liu2026neural,wu2026mixture,liu2026multi,liu2026neural2,liu2023knowledge2,liu2025grapho1}
 . 
Recent studies show that standard CP can fail to maintain coverage when inputs are adversarially perturbed \cite{gendler2022adversarially}. A small number of works have explored robust variants of CP. Early approaches lack formal guarantees \cite{hechtlinger2018cautious}, while more recent methods, such as Randomly Smoothed Conformal Prediction (RSCP), combine CP with randomized smoothing to provide robustness guarantees \cite{gendler2022adversarially}. Although effective in theory, RSCP relies on inflating conformal quantiles in proportion to the perturbation radius, often resulting in overly conservative and large prediction sets. Overall, robust conformal prediction remains an underexplored area, especially in complex decision-making settings.

\paragraph{Uncertainty Quantification for LLM-as-a-Judge.}
Uncertainty quantification for LLM-based judges is an emerging but still limited area of research. Existing approaches typically rely on token-level probabilities \cite{wagner2024black, xie2025empirical}, self-reported confidence \citep{yona2024can, xu2024sayself}, or consistency across multiple generations \citep{tian2023just}. These methods often suffer from bias, instability, overconfidence, or high computational cost.

Conformal prediction has recently attracted attention as a principled, post-hoc uncertainty quantification tool for LLMs due to its distribution-free guarantees \citep{ye2024benchmarking, campos2024conformal}. Most existing work applies CP to classification-style tasks, such as multiple-choice question answering or response selection, where the goal is to construct unordered prediction sets that include the correct answer with high probability. The most closely related work to ours applies conformalized risk control to align LLM judgments with human preferences in pairwise comparisons \citep{jung2024trust}.

\section{Conclusion}

This work addresses a key limitation of LLM-as-a-Judge evaluation by moving beyond uncertainty estimates derived from a single LLM. We propose a robust, multi-LLM interval adaptation framework that integrates seamlessly with existing conformal prediction methods and explicitly accounts for inter-LLM variability in rating-based evaluation tasks. Through extensive experiments across multiple LLMs and conformal techniques, we show that incorporating multiple judges produces more stable prediction intervals with improved robustness, while maintaining strong coverage guarantees.

\section{Ethical Considerations}

We have carefully assessed the potential risks associated with our work and do not foresee any significant ethical concerns. Our framework is designed with a strong emphasis on usability and ease of implementation, lowering adoption barriers while minimizing operational complexities. Additionally, our research is built upon an open-source dataset, ensuring transparency, fostering collaboration, and promoting ethical integrity by providing accessible and reproducible data.

\section{Limitations}

Despite its effectiveness, our approach has certain limitations. First, our training data is limited in scope, primarily covering specific domains rather than offering broad generalization across diverse topics. This constraint may impact the model's performance when handling queries beyond these predefined areas. Furthermore, while our knowledge graph provides valuable contextual information, it remains inherently incomplete. Certain regions may lack sufficient data or relational links, potentially leading to gaps in the model’s reasoning and inference capabilities.

\bibliographystyle{acl_natbib}
\bibliography{008reference.bib,liu}

\begin{thebibliography}{70}
\expandafter\ifx\csname natexlab\endcsname\relax\def\natexlab#1{#1}\fi

\bibitem[{Barbosa et~al.(2025)Barbosa, Gondhali, Petrossian, Sharma, Chakraborty, Jacquet, and Freire}]{traffic}
Juliana~Silva Barbosa, Ulhas Gondhali, Gohar Petrossian, Kinshuk Sharma, Sunandan Chakraborty, Jennifer Jacquet, and Juliana Freire. 2025.
\newblock A cost-effective llm-based approach to identify wildlife trafficking in online marketplaces.
\newblock \emph{Proceedings of the ACM on Management of Data}, 3(3):1--23.

\bibitem[{Bedi et~al.(2025)Bedi, Liu, Orr-Ewing, Dash, Koyejo, Callahan, Fries, Wornow, Swaminathan, Lehmann et~al.}]{healthcare}
Suhana Bedi, Yutong Liu, Lucy Orr-Ewing, Dev Dash, Sanmi Koyejo, Alison Callahan, Jason~A Fries, Michael Wornow, Akshay Swaminathan, Lisa~Soleymani Lehmann, et~al. 2025.
\newblock Testing and evaluation of health care applications of large language models: a systematic review.
\newblock \emph{Jama}.

\bibitem[{Camburu et~al.(2018)Camburu, Rockt\"{a}schel, Lukasiewicz, and Blunsom}]{e_SNLI}
Oana-Maria Camburu, Tim Rockt\"{a}schel, Thomas Lukasiewicz, and Phil Blunsom. 2018.
\newblock e-snli: natural language inference with natural language explanations.
\newblock In \emph{Proceedings of the 32nd International Conference on Neural Information Processing Systems}, NIPS'18, page 9560–9572, Red Hook, NY, USA. Curran Associates Inc.

\bibitem[{Campos et~al.(2024)Campos, Farinhas, Zerva, Figueiredo, and Martins}]{campos2024conformal}
Margarida~M Campos, Ant{\'o}nio Farinhas, Chrysoula Zerva, M{\'a}rio~AT Figueiredo, and Andr{\'e}~FT Martins. 2024.
\newblock Conformal prediction for natural language processing: A survey.
\newblock \emph{arXiv preprint arXiv:2405.01976}.

\bibitem[{DeepSeek-AI(2025)}]{deepseekai2025deepseekr1incentivizingreasoningcapability}
DeepSeek-AI. 2025.
\newblock \href {http://arxiv.org/abs/2501.12948} {Deepseek-r1: Incentivizing reasoning capability in llms via reinforcement learning}.

\bibitem[{Dua et~al.(2019)Dua, Wang, Dasigi, Stanovsky, Singh, and Gardner}]{dua-etal-2019-drop}
Dheeru Dua, Yizhong Wang, Pradeep Dasigi, Gabriel Stanovsky, Sameer Singh, and Matt Gardner. 2019.
\newblock {DROP}: A reading comprehension benchmark requiring discrete reasoning over paragraphs.
\newblock In \emph{Proceedings of the 2019 Conference of the North {A}merican Chapter of the Association for Computational Linguistics: Human Language Technologies, Volume 1 (Long and Short Papers)}, Minneapolis, Minnesota. Association for Computational Linguistics.

\bibitem[{Fabbri et~al.(2021)Fabbri, Kryściński, McCann, Xiong, Socher, and Radev}]{SummEval}
Alexander~R. Fabbri, Wojciech Kryściński, Bryan McCann, Caiming Xiong, Richard Socher, and Dragomir Radev. 2021.
\newblock \href {http://arxiv.org/abs/2007.12626} {Summeval: Re-evaluating summarization evaluation}.

\bibitem[{Gendler et~al.(2022)Gendler, Weng, Daniel, and Romano}]{gendler2022adversarially}
Asaf Gendler, Tsui-Wei Weng, Luca Daniel, and Yaniv Romano. 2022.
\newblock Adversarially robust conformal prediction.
\newblock In \emph{International Conference on Learning Representations (ICLR)}.

\bibitem[{Golovneva et~al.(2023)Golovneva, Chen, Poff, Corredor, Zettlemoyer, Fazel-Zarandi, and Celikyilmaz}]{golovneva2023roscoesuitemetricsscoring}
Olga Golovneva, Moya Chen, Spencer Poff, Martin Corredor, Luke Zettlemoyer, Maryam Fazel-Zarandi, and Asli Celikyilmaz. 2023.
\newblock \href {http://arxiv.org/abs/2212.07919} {Roscoe: A suite of metrics for scoring step-by-step reasoning}.

\bibitem[{Goo and Chen(2018)}]{DialSumm}
Chih-Wen Goo and Yun-Nung Chen. 2018.
\newblock Abstractive dialogue summarization with sentence-gated modeling optimized by dialogue acts.
\newblock In \emph{Proceedings of 7th IEEE Workshop on Spoken Language Technology}.

\bibitem[{Gu et~al.(2025)Gu, Jiang, Shi, Tan, Zhai, Xu, Li, Shen, Ma, Liu, Wang, Zhang, Wang, Gao, Ni, and Guo}]{unstable}
Jiawei Gu, Xuhui Jiang, Zhichao Shi, Hexiang Tan, Xuehao Zhai, Chengjin Xu, Wei Li, Yinghan Shen, Shengjie Ma, Honghao Liu, Saizhuo Wang, Kun Zhang, Yuanzhuo Wang, Wen Gao, Lionel Ni, and Jian Guo. 2025.
\newblock \href {http://arxiv.org/abs/2411.15594} {A survey on llm-as-a-judge}.

\bibitem[{Guha et~al.(2024)Guha, Natarajan, M{\"o}llenhoff, Khan, and Ndiaye}]{r2cpp}
Etash Guha, Shlok Natarajan, Thomas M{\"o}llenhoff, Mohammad~Emtiyaz Khan, and Eugene Ndiaye. 2024.
\newblock Conformal prediction via regression-as-classification.
\newblock \emph{arXiv preprint arXiv:2404.08168}.

\bibitem[{He et~al.(2024)He, Zhang, and Roth}]{he-etal-2024-socreval}
Hangfeng He, Hongming Zhang, and Dan Roth. 2024.
\newblock \href {https://doi.org/10.18653/v1/2024.findings-naacl.175} {{S}oc{RE}val: Large language models with the socratic method for reference-free reasoning evaluation}.
\newblock In \emph{Findings of the Association for Computational Linguistics: NAACL 2024}, pages 2736--2764. Association for Computational Linguistics.

\bibitem[{Hechtlinger et~al.(2018)Hechtlinger, P{\'o}czos, and Wasserman}]{hechtlinger2018cautious}
Yotam Hechtlinger, Barnab{\'a}s P{\'o}czos, and Larry Wasserman. 2018.
\newblock Cautious deep learning.
\newblock \emph{arXiv preprint arXiv:1805.09460}.

\bibitem[{Huang et~al.(2019)Huang, Bras, Bhagavatula, and Choi}]{huang2019cosmosqamachinereading}
Lifu Huang, Ronan~Le Bras, Chandra Bhagavatula, and Yejin Choi. 2019.
\newblock \href {http://arxiv.org/abs/1909.00277} {Cosmos qa: Machine reading comprehension with contextual commonsense reasoning}.

\bibitem[{Jung et~al.(2024)Jung, Brahman, and Choi}]{jung2024trust}
Jaehun Jung, Faeze Brahman, and Yejin Choi. 2024.
\newblock Trust or escalate: Llm judges with provable guarantees for human agreement.

\bibitem[{Lin(2004)}]{lin-2004-rouge}
Chin-Yew Lin. 2004.
\newblock \href {https://aclanthology.org/W04-1013/} {{ROUGE}: A package for automatic evaluation of summaries}.
\newblock In \emph{Text Summarization Branches Out}, pages 74--81, Barcelona, Spain. Association for Computational Linguistics.

\bibitem[{Lin et~al.(2021)Lin, Trivedi, and Sun}]{LVD}
Zhen Lin, Shubhendu Trivedi, and Jimeng Sun. 2021.
\newblock Locally valid and discriminative prediction intervals for deep learning models.
\newblock \emph{Advances in Neural Information Processing Systems}, 34:8378--8391.

\bibitem[{Liu(2024)}]{liu2024knowledge}
Lihui Liu. 2024.
\newblock \emph{Knowledge graph reasoning and its applications: A pathway towards neural symbolic AI}.
\newblock Ph.D. thesis, University of Illinois at Urbana-Champaign.

\bibitem[{Liu(2025{\natexlab{a}})}]{liu2025grapho1}
Lihui Liu. 2025{\natexlab{a}}.
\newblock Graph-o1: Monte carlo tree search with reinforcement learning for text-attributed graph reasoning.
\newblock \emph{arXiv preprint arXiv:2512.17912}.

\bibitem[{Liu(2025{\natexlab{b}})}]{liu2025hyperkgr}
Lihui Liu. 2025{\natexlab{b}}.
\newblock Hyperkgr: Knowledge graph reasoning in hyperbolic space with graph neural network encoding symbolic path.
\newblock In \emph{Proceedings of the 2025 Conference on Empirical Methods in Natural Language Processing}, pages 25188--25199.

\bibitem[{Liu(2025{\natexlab{c}})}]{liu2025monte}
Lihui Liu. 2025{\natexlab{c}}.
\newblock Monte carlo tree search for graph reasoning in large language model agents.
\newblock In \emph{Proceedings of the 34th ACM International Conference on Information and Knowledge Management}, pages 4966--4970.

\bibitem[{Liu(2026{\natexlab{a}})}]{liu2026multi}
Lihui Liu. 2026{\natexlab{a}}.
\newblock Multi-hop reasoning and retrieval in embedding space: Leveraging large language models with knowledge.
\newblock \emph{arXiv preprint arXiv:2603.13266}.

\bibitem[{Liu(2026{\natexlab{b}})}]{liu2026neural2}
Lihui Liu. 2026{\natexlab{b}}.
\newblock Neural-symbolic logic query answering in non-euclidean space.
\newblock \emph{arXiv preprint arXiv:2603.15633}.

\bibitem[{Liu et~al.(2023{\natexlab{a}})Liu, Chen, Das, Yang, and Tong}]{liu2023knowledge}
Lihui Liu, Yuzhong Chen, Mahashweta Das, Hao Yang, and Hanghang Tong. 2023{\natexlab{a}}.
\newblock Knowledge graph question answering with ambiguous query.
\newblock In \emph{Proceedings of the ACM Web Conference 2023}, pages 2477--2486.

\bibitem[{Liu et~al.(2025{\natexlab{a}})Liu, Ding, Mukherjee, and Yang}]{liu2025mixrag}
Lihui Liu, Jiayuan Ding, Subhabrata Mukherjee, and Carl~J Yang. 2025{\natexlab{a}}.
\newblock Mixrag: Mixture-of-experts retrieval-augmented generation for textual graph understanding and question answering.
\newblock \emph{arXiv preprint arXiv:2509.21391}.

\bibitem[{Liu et~al.(2021{\natexlab{a}})Liu, Du, Fung, Ji, Xu, and Tong}]{liu2021kompare}
Lihui Liu, Boxin Du, Yi~Ren Fung, Heng Ji, Jiejun Xu, and Hanghang Tong. 2021{\natexlab{a}}.
\newblock Kompare: A knowledge graph comparative reasoning system.
\newblock In \emph{Proceedings of the 27th ACM SIGKDD Conference on Knowledge Discovery \& Data Mining}, pages 3308--3318.

\bibitem[{Liu et~al.(2021{\natexlab{b}})Liu, Du, Ji, Zhai, and Tong}]{liu2021neural}
Lihui Liu, Boxin Du, Heng Ji, ChengXiang Zhai, and Hanghang Tong. 2021{\natexlab{b}}.
\newblock Neural-answering logical queries on knowledge graphs.
\newblock In \emph{Proceedings of the 27th ACM SIGKDD Conference on Knowledge Discovery \& Data Mining}, pages 1087--1097.

\bibitem[{Liu et~al.(2019)Liu, Du, Tong et~al.}]{liu2019g}
Lihui Liu, Boxin Du, Hanghang Tong, et~al. 2019.
\newblock G-finder: Approximate attributed subgraph matching.
\newblock In \emph{2019 IEEE International Conference on Big Data (Big Data)}, pages 513--522. IEEE.

\bibitem[{Liu et~al.(2022{\natexlab{a}})Liu, Du, Xu, Xia, and Tong}]{liu2022joint}
Lihui Liu, Boxin Du, Jiejun Xu, Yinglong Xia, and Hanghang Tong. 2022{\natexlab{a}}.
\newblock Joint knowledge graph completion and question answering.
\newblock In \emph{Proceedings of the 28th ACM SIGKDD Conference on Knowledge Discovery and Data Mining}, pages 1098--1108.

\bibitem[{Liu et~al.(2024{\natexlab{a}})Liu, Hill, Du, Wang, and Tong}]{liu2024conversational}
Lihui Liu, Blaine Hill, Boxin Du, Fei Wang, and Hanghang Tong. 2024{\natexlab{a}}.
\newblock Conversational question answering with language models generated reformulations over knowledge graph.
\newblock In \emph{Findings of the Association for Computational Linguistics ACL 2024}, pages 839--850.

\bibitem[{Liu et~al.(2022{\natexlab{b}})Liu, Ji, Xu, and Tong}]{liu2022comparative}
Lihui Liu, Houxiang Ji, Jiejun Xu, and Hanghang Tong. 2022{\natexlab{b}}.
\newblock Comparative reasoning for knowledge graph fact checking.
\newblock In \emph{2022 IEEE International Conference on Big Data (Big Data)}, pages 2309--2312. IEEE.

\bibitem[{Liu et~al.(2024{\natexlab{b}})Liu, Kim, and Bansal}]{liu2024can}
Lihui Liu, Jinha Kim, and Vidit Bansal. 2024{\natexlab{b}}.
\newblock Can contrastive learning refine embeddings.
\newblock \emph{arXiv preprint arXiv:2404.08701}.

\bibitem[{Liu and Shu(2025)}]{liu2025unifying}
Lihui Liu and Kai Shu. 2025.
\newblock Unifying knowledge in agentic llms: Concepts, methods, and recent advancements.
\newblock \emph{ACM SIGKDD Explorations Newsletter}, 27(2):88--96.

\bibitem[{Liu and Tong()}]{liuneural}
Lihui Liu and Hanghang Tong.
\newblock Neural symbolic knowledge graph reasoning.

\bibitem[{Liu and Tong(2023)}]{liu2023knowledge2}
Lihui Liu and Hanghang Tong. 2023.
\newblock Knowledge graph reasoning and its applications.
\newblock In \emph{Proceedings of the 29th ACM SIGKDD Conference on Knowledge Discovery and Data Mining}, pages 5813--5814.

\bibitem[{Liu and Tong(2026{\natexlab{a}})}]{liu2026accurate}
Lihui Liu and Hanghang Tong. 2026{\natexlab{a}}.
\newblock Accurate query answering with symbolic reasoning over complete kgs.
\newblock In \emph{Neural Symbolic Knowledge Graph Reasoning: A Pathway Towards Neural Symbolic AI}, pages 17--36. Springer.

\bibitem[{Liu and Tong(2026{\natexlab{b}})}]{liu2026ambiguous}
Lihui Liu and Hanghang Tong. 2026{\natexlab{b}}.
\newblock Ambiguous query answering with neural symbolic reasoning over incomplete kg.
\newblock In \emph{Neural Symbolic Knowledge Graph Reasoning: A Pathway Towards Neural Symbolic AI}, pages 89--106. Springer.

\bibitem[{Liu and Tong(2026{\natexlab{c}})}]{liu2026dynamic}
Lihui Liu and Hanghang Tong. 2026{\natexlab{c}}.
\newblock Dynamic query answering with neural symbolic reasoning over incomplete kg.
\newblock In \emph{Neural Symbolic Knowledge Graph Reasoning: A Pathway Towards Neural Symbolic AI}, pages 121--136. Springer.

\bibitem[{Liu and Tong(2026{\natexlab{d}})}]{liu2026neural}
Lihui Liu and Hanghang Tong. 2026{\natexlab{d}}.
\newblock \emph{Neural Symbolic Knowledge Graph Reasoning: A Pathway Towards Neural Symbolic AI}.
\newblock Springer Nature.

\bibitem[{Liu and Tong(2026{\natexlab{e}})}]{liu2026symbolic}
Lihui Liu and Hanghang Tong. 2026{\natexlab{e}}.
\newblock Symbolic reasoning for inconsistency detection over complete kg.
\newblock In \emph{Neural Symbolic Knowledge Graph Reasoning: A Pathway Towards Neural Symbolic AI}, pages 37--54. Springer.

\bibitem[{Liu et~al.(2024{\natexlab{c}})Liu, Wang, Bai, Song, and Tong}]{liu2024new}
Lihui Liu, Zihao Wang, Jiaxin Bai, Yangqiu Song, and Hanghang Tong. 2024{\natexlab{c}}.
\newblock New frontiers of knowledge graph reasoning: Recent advances and future trends.
\newblock In \emph{Companion Proceedings of the ACM Web Conference 2024}, pages 1294--1297.

\bibitem[{Liu et~al.(2024{\natexlab{d}})Liu, Wang, Qiu, Ban, Chan, Song, He, and Tong}]{liu2024logic}
Lihui Liu, Zihao Wang, Ruizhong Qiu, Yikun Ban, Eunice Chan, Yangqiu Song, Jingrui He, and Hanghang Tong. 2024{\natexlab{d}}.
\newblock Logic query of thoughts: Guiding large language models to answer complex logic queries with knowledge graphs.
\newblock \emph{arXiv preprint arXiv:2404.04264}.

\bibitem[{Liu et~al.(2025{\natexlab{b}})Liu, Wang, and Tong}]{liu2025neural}
Lihui Liu, Zihao Wang, and Hanghang Tong. 2025{\natexlab{b}}.
\newblock Neural-symbolic reasoning over knowledge graphs: A survey from a query perspective.
\newblock \emph{ACM SIGKDD Explorations Newsletter}, 27(1):124--136.

\bibitem[{Liu et~al.(2025{\natexlab{c}})Liu, Wang, Zhou, Wang, Yan, Xiong, He, and Tong}]{liu2025few}
Lihui Liu, Zihao Wang, Dawei Zhou, Ruijie Wang, Yuchen Yan, Bo~Xiong, Sihong He, and Hanghang Tong. 2025{\natexlab{c}}.
\newblock Few-shot knowledge graph completion via transfer knowledge from similar tasks.
\newblock In \emph{Proceedings of the 34th ACM International Conference on Information and Knowledge Management}, pages 4960--4965.

\bibitem[{Liu and Yan(2026)}]{liu2026morgan}
Lihui Liu and Yuchen Yan. 2026.
\newblock Morgan: To bridge mixture of experts and spectral graph neural network.
\newblock In \emph{Proceedings of the AAAI Conference on Artificial Intelligence}, volume~40, pages 23783--23791.

\bibitem[{Liu and Yang(2026)}]{liu2026prompt}
Lihui Liu and Carl Yang. 2026.
\newblock Prompt-tuning with attribute guidance for low-resource entity matching.
\newblock \emph{arXiv preprint arXiv:2603.19321}.

\bibitem[{Liu et~al.(2022{\natexlab{c}})Liu, Zhao, Du, Fung, Ji, Xu, and Tong}]{liu2022knowledge}
Lihui Liu, Ruining Zhao, Boxin Du, Yi~Ren Fung, Heng Ji, Jiejun Xu, and Hanghang Tong. 2022{\natexlab{c}}.
\newblock Knowledge graph comparative reasoning for fact checking: Problem definition and algorithms.
\newblock \emph{IEEE Data Eng. Bull.}, 45(4):19--38.

\bibitem[{Liu et~al.(2023{\natexlab{b}})Liu, Iter, Xu, Wang, Xu, and Zhu}]{liu-etal-2023-g}
Yang Liu, Dan Iter, Yichong Xu, Shuohang Wang, Ruochen Xu, and Chenguang Zhu. 2023{\natexlab{b}}.
\newblock {G}-eval: {NLG} evaluation using gpt-4 with better human alignment.
\newblock In \emph{Proceedings of the 2023 Conference on Empirical Methods in Natural Language Processing}, Singapore. Association for Computational Linguistics.

\bibitem[{Lu et~al.(2022)Lu, Angelopoulos, and Pomerantz}]{aps}
Charles Lu, Anastasios~N Angelopoulos, and Stuart Pomerantz. 2022.
\newblock Improving trustworthiness of ai disease severity rating in medical imaging with ordinal conformal prediction sets.
\newblock In \emph{International conference on medical image computing and computer-assisted intervention}, pages 545--554. Springer.

\bibitem[{Papineni et~al.(2002)Papineni, Roukos, Ward, and Zhu}]{blue}
Kishore Papineni, Salim Roukos, Todd Ward, and Wei-Jing Zhu. 2002.
\newblock Bleu: a method for automatic evaluation of machine translation.
\newblock In \emph{Proceedings of the 40th Annual Meeting on Association for Computational Linguistics}, ACL '02, page 311–318, USA. Association for Computational Linguistics.

\bibitem[{Qwen et~al.(2025)Qwen, :, Yang, Yang, and Zhang}]{qwen2025qwen25technicalreport}
Qwen, :, An~Yang, Baosong Yang, and Beichen Zhang. 2025.
\newblock \href {http://arxiv.org/abs/2412.15115} {Qwen2.5 technical report}.

\bibitem[{Romano et~al.(2019)Romano, Patterson, and Candes}]{CQR}
Yaniv Romano, Evan Patterson, and Emmanuel Candes. 2019.
\newblock Conformalized quantile regression.
\newblock \emph{Advances in neural information processing systems}, 32.

\bibitem[{Sesia and Romano(2021)}]{CHR}
Matteo Sesia and Yaniv Romano. 2021.
\newblock Conformal prediction using conditional histograms.
\newblock \emph{Advances in neural information processing systems}, 34:6304--6315.

\bibitem[{Shafer and Vovk(2008{\natexlab{a}})}]{conformal_prediction}
Glenn Shafer and Vladimir Vovk. 2008{\natexlab{a}}.
\newblock A tutorial on conformal prediction.
\newblock \emph{Journal of Machine Learning Research}, 9(3).

\bibitem[{Shafer and Vovk(2008{\natexlab{b}})}]{shafer2008tutorial}
Glenn Shafer and Vladimir Vovk. 2008{\natexlab{b}}.
\newblock A tutorial on conformal prediction.
\newblock \emph{Journal of Machine Learning Research}, 9(3).

\bibitem[{Sheng et~al.(2025)Sheng, Liu, He, Zhao, and Kang}]{Jian}
Huanxin Sheng, Xinyi Liu, Hangfeng He, Jieyu Zhao, and Jian Kang. 2025.
\newblock Analyzing uncertainty of {LLM}-as-a-judge: Interval evaluations with conformal prediction.
\newblock In \emph{Proceedings of the 2025 Conference on Empirical Methods in Natural Language Processing}, Suzhou, China. Association for Computational Linguistics.

\bibitem[{Tian et~al.(2023)Tian, Mitchell, Zhou, Sharma, Rafailov, Yao, Finn, and Manning}]{tian2023just}
Katherine Tian, Eric Mitchell, Allan Zhou, Archit Sharma, Rafael Rafailov, Huaxiu Yao, Chelsea Finn, and Christopher~D Manning. 2023.
\newblock Just ask for calibration: Strategies for eliciting calibrated confidence scores from language models fine-tuned with human feedback.

\bibitem[{Wagner et~al.(2024)Wagner, Desmond, Nair, Ashktorab, Daly, Pan, Cooper, Johnson, and Geyer}]{wagner2024black}
Nico Wagner, Michael Desmond, Rahul Nair, Zahra Ashktorab, Elizabeth~M Daly, Qian Pan, Martin~Santillan Cooper, James~M Johnson, and Werner Geyer. 2024.
\newblock Black-box uncertainty quantification method for llm-as-a-judge.

\bibitem[{Wu and Liu(2026)}]{wu2026mixture}
Yukun Wu and Lihui Liu. 2026.
\newblock Mixture of demonstrations for textual graph understanding and question answering.
\newblock \emph{arXiv preprint arXiv:2603.23554}.

\bibitem[{Xie et~al.(2025)Xie, Li, Yu, Zhang, Zhang, and Yang}]{xie2025empirical}
Qiujie Xie, Qingqiu Li, Zhuohao Yu, Yuejie Zhang, Yue Zhang, and Linyi Yang. 2025.
\newblock An empirical analysis of uncertainty in large language model evaluations.
\newblock \emph{arXiv preprint}.

\bibitem[{Xie et~al.(2024)Xie, Barber, and Candes}]{boost}
Ran Xie, Rina Barber, and Emmanuel Candes. 2024.
\newblock Boosted conformal prediction intervals.
\newblock \emph{Advances in Neural Information Processing Systems}, 37:71868--71899.

\bibitem[{Xu et~al.(2024)Xu, Wu, Diao, Liu, Wang, Chen, and Gao}]{xu2024sayself}
Tianyang Xu, Shujin Wu, Shizhe Diao, Xiaoze Liu, Xingyao Wang, Yangyi Chen, and Jing Gao. 2024.
\newblock Sayself: Teaching llms to express confidence with self-reflective rationales.

\bibitem[{Xu et~al.(2023)Xu, Guo, and Wei}]{ordinal_risk_control}
Yunpeng Xu, Wenge Guo, and Zhi Wei. 2023.
\newblock Conformal risk control for ordinal classification.
\newblock In \emph{Uncertainty in Artificial Intelligence}, pages 2346--2355. PMLR.

\bibitem[{Ye et~al.(2024)Ye, Yang, Pang, Wang, Wong, Yilmaz, Shi, and Tu}]{ye2024benchmarking}
Fanghua Ye, Mingming Yang, Jianhui Pang, Longyue Wang, Derek~F Wong, Emine Yilmaz, Shuming Shi, and Zhaopeng Tu. 2024.
\newblock Benchmarking llms via uncertainty quantification.

\bibitem[{Yona et~al.(2024)Yona, Aharoni, and Geva}]{yona2024can}
Gal Yona, Roee Aharoni, and Mor Geva. 2024.
\newblock Can large language models faithfully express their intrinsic uncertainty in words?

\bibitem[{Yuan and He(2024)}]{safety}
Tongxin Yuan and Zhiwei He. 2024.
\newblock R-judge: Benchmarking safety risk awareness for llm agents.
\newblock \emph{arXiv preprint arXiv:2401.10019}.

\bibitem[{Zeng et~al.(2024)Zeng, Chen, Liu, Jiang, and Jia}]{zeng2024mrgsm8kmetareasoningbenchmarklarge}
Zhongshen Zeng, Pengguang Chen, Shu Liu, Haiyun Jiang, and Jiaya Jia. 2024.
\newblock \href {http://arxiv.org/abs/2312.17080} {Mr-gsm8k: A meta-reasoning benchmark for large language model evaluation}.

\bibitem[{Zhang et~al.(2020)Zhang, Kishore, Wu, Weinberger, and Artzi}]{BERTScore}
Tianyi Zhang, Varsha Kishore, Felix Wu, Kilian~Q. Weinberger, and Yoav Artzi. 2020.
\newblock \href {http://arxiv.org/abs/1904.09675} {Bertscore: Evaluating text generation with bert}.

\bibitem[{Zheng and Chiang(2023)}]{qa}
Lianmin Zheng and Wei-Lin Chiang. 2023.
\newblock Judging llm-as-a-judge with mt-bench and chatbot arena.
\newblock \emph{Advances in neural information processing systems}, 36:46595--46623.

\end{thebibliography}

\clearpage

\end{document}